\documentclass[journal]{IEEEtran}
\usepackage{ifpdf} 
\usepackage{cite} 
\usepackage{color} 
\usepackage{array} 
\usepackage{url} 
\usepackage{multirow}
\usepackage{array}
\usepackage{xcolor}
\usepackage{amssymb}
\usepackage{tikz}
\usepackage{caption}
\usepackage{physics}
\usepackage{graphicx}
\usepackage{yhmath}
\usepackage{mathdots}
\usepackage{MnSymbol}
\usepackage{lipsum}
\usepackage{svg}
\usepackage{float}
\usepackage{bbm}
\usepackage[export]{adjustbox}
\usepackage{hyperref}
\usepackage[english]{babel}
\usepackage{cuted}
\usepackage{cellspace} 
\usepackage{bbm}
\usepackage{amsthm}
\usepackage{amssymb}

\usepackage{amsmath} 

\allowdisplaybreaks

\usepackage[caption=false,font=normalsize,labelfont=sf,textfont=sf]{subfig}

\usepackage{eqparbox}
\usepackage{arydshln}
\usepackage{float}

\newtheorem{theorem}{Theorem}[section]

\newtheorem{corollary}{Corollary}

\begin{document}
\title{L-Lipschitz Gershgorin ResNet Network }

\author{Marius~F.~R.~Juston$^{1}$,
        William~R.~Norris$^{2}$,
        Dustin~Nottage$^{3}$,
        Ahmet~Soylemezoglu$^{3}$

\thanks{Marius F. R. Juston$^{1}$ is with The Grainger College of Engineering, Industrial and Enterprise Systems Engineering Department, University of Illinois Urbana-Champaign, Urbana, IL 61801-3080 USA (email: mjuston2@illinois.edu).}


\thanks{William R Norris$^{2}$ is with The Grainger College of Engineering, Industrial and Enterprise Systems Engineering Department, University of Illinois Urbana-Champaign, Urbana, IL 61801-3080 USA (email: wrnorris@illinois.edu).}

\thanks{Construction Engineering Research Laboratory$^{3}$, U.S. Army Corps of Engineers Engineering Research and Development Center, IL, 61822, USA}

\thanks{This research was supported by the U.S. Army Corps of Engineers Engineering Research and Development Center, Construction Engineering Research Laboratory.}
}

\markboth{IEEE TRANSACTIONS ON XXX XXXX, XXX, XXX, September~2023}%
{Shell \MakeLowercase{\textit{et al.}}: Bare Demo of IEEEtran.cls for IEEE Journals}


\maketitle

\begin{abstract}
Deep residual networks (ResNets) have demonstrated outstanding success in computer vision tasks, attributed to their ability to maintain gradient flow through deep architectures. Simultaneously, controlling the Lipschitz bound in neural networks has emerged as an essential area of research for enhancing adversarial robustness and network certifiability. This paper uses a rigorous approach to design $\mathcal{L}$-Lipschitz deep residual networks using a Linear Matrix Inequality (LMI) framework. The ResNet architecture was reformulated as a pseudo-tri-diagonal LMI with off-diagonal elements and derived closed-form constraints on network parameters to ensure $\mathcal{L}$-Lipschitz continuity. To address the lack of explicit eigenvalue computations for such matrix structures, the Gershgorin circle theorem was employed to approximate eigenvalue locations, guaranteeing the LMI's negative semi-definiteness. Our contributions include a provable parameterization methodology for constructing Lipschitz-constrained networks and a compositional framework for managing recursive systems within hierarchical architectures. These findings enable robust network designs applicable to adversarial robustness, certified training, and control systems. However, a limitation was identified in the Gershgorin-based approximations, which over-constrain the system, suppressing non-linear dynamics and diminishing the network's expressive capacity. 
\end{abstract}


\IEEEpeerreviewmaketitle

\section{Introduction} \label{s_Intro}

\IEEEPARstart{T}{he} robustness of deep neural networks (DNNs) is a critical challenge, mainly when applied in safety-sensitive domains where small adversarial perturbations can lead to dangerous situations such as the misclassification of important objects. One approach to address this issue is by enforcing Lipschitz constraints on the network architectures. These constraints guarantee that small changes in the input will not cause significant changes in the output. This property is vital for certifying robustness against adversarial attacks, which involve introducing slight noise to modify the expected classification output result \cite{Inkawhich2019, Goodfellow2014}. The Lipschitz constant is a key measure to bound the network's sensitivity to input perturbations. Specifically, a $\mathcal{L}$-Lipschitz network can be theoretically guaranteed that the output remains stable within a defined "stability sphere" around each input, making it resistant to adversarial attacks up to a certain magnitude \cite{Tsuzuku2018}.

To achieve this, several methods have been proposed to enforce Lipschitz constraints on neural networks, including spectral normalization \cite{Miyato2018, Bartlett2017}, orthogonal parameterization \cite{Prach2022}, and more recent approaches such as Convex Potential Layers (CPL) and Almost-Orthogonal Layers (AOL) \cite{Meunier2022, Prach2022}. The previous works have been shown to be formulated under a unifying semi-definitive programming architecture, which possesses the constraints on the networks as LMIs \cite{Araujo2023}. However, ensuring Lipschitz constraints in deep architectures, particularly residual networks (ResNets), presents unique challenges due to their recursive structure. While prior work has made strides in constraining individual layers \cite{Araujo2023, Meunier2021} and generating a unifying semi-definite programming approach, the generalized deep residual network formulation presents issues in the pseudo-tri-diagonal structure of its imposed LMI. 

Furthermore, multi-layered general Feedforward Neural Networks (FNN) have been shown to generate block tri-diagonal matrix LMI formulations \cite{Xu2024} due to their inherent network structure, which in contrast to the residual formulation, yield explicit solutions \cite{Sandryhaila2013, Agarwal2019}. However, due to the off-diagonal structure of the network, the direct application of the exact eigenvalue computation is not feasible, making the solution process significantly more complex.

Previous work has also demonstrated an iterative approach by utilizing projected gradient descent optimization or a regularization term on the estimated Lipschitz constant to ensure a constraint on the Lipschitz constraint \cite{Gouk2021, Aziznejad2020, Bear2024}. While this guarantees an iterative enforcement of the Lipschitz constraint, it does not ensure a theoretical Lipschitz guarantee across the entire network until this convergence. However, the advantage of this technique is its generalizability, which allows for the utilization of more general network structures.

\subsection{Contributions}

This paper introduces the formulation of deep residual networks as Linear Matrix Inequalities (LMI). It derives closed-form constraints on network parameters to ensure theoretical $\mathcal{L}$-Lipschitz constraints. The LMI was structured as a tri-diagonal matrix with off-diagonal components, which inherently complicates the derivation of closed-form eigenvalue computations. To address this limitation, the Gershgorin circle theorem was employed to approximate the eigenvalue locations. The Gershgorin circles enabled the derivation of closed-form constraints that guaranteed the negative semi-definiteness of the LMI. 

Additionally, this paper demonstrates a significant limitation of the Gershgorin circle theorem in this context: the derived approximations lead to over-constraining the system, effectively suppressing the network's non-linear components. This, in turn, makes the network act as a simple linear transformation instead. 

Moreover, while \cite{Araujo2023}'s work generates a closed-form solution for a residual network, it is limited to considering a single inner layer. In contrast, this paper presents a more general formulation that accommodates a more expressive inner layers system within the residual network system, offering greater flexibility and broader applicability.

\section{LMI Formulation}

Following the works for \cite{Araujo2023}, who defined a Lipschitz neural network as a constrained LMI problem to define a residual network, limitations in their approach were identified. Specifically, their formulation resulted in a single-layered residual network, which is inherently less expressive compared to the generalized deep-layered residual network popularized by architectures such as ResNet and its variants \cite{He2015ResNet, Szegedy2016, Zagoruyko2016, Hu2017, Xie2016}. These deeper networks perform better due to the multiple inner layers that compose the modules, which allows for more complex latent space transformations and thus increases the network's expressiveness. This research focuses on establishing constraints for the inner layers to maintain the $\mathcal{L}$-Lipschitz condition while maximizing the expressiveness of the residual network for larger inner layers. As such, the inner layers of the residual network were represented as a recursive system of linear equations:
\begin{align}
    x_{k + 1} &= A_k x_k + B_k w_{k,n} \nonumber \\
    w_{k,n} &=  \sigma_n(C_n w_{k,n - 1} + b_n ) \nonumber \\
    & \vdots \nonumber \\
    w_{k,1} &=  \sigma_1(C_1 x_{k} + b_1 ).
\end{align}
Where each of the layer parameters were defined as $C_l \in \mathbb{R}^{d_l \times d_{l - 1}}, b_l \in \mathbb{R}^{d_l}$ for $l \in \{1,\cdots,n\}$. When $n= 1$, the formulation reduced to the one presented in \cite{Araujo2023}, rendering it redundant in its derivation. The goal of the LMI was to maintain the Lipschitz constraint formulated as $\norm{x'_{k + 1} - x_{k + 1}} \leq \mathcal{L}\norm{x'_{k} - x_{k}}$.

Given that this system could be represented as a large recursive system, it was possible to split all the constraints of the inner layers as a set of LMI conditions similar to \cite{Araujo2023, Xu2024, 10.5555/3454287.3455312}. For the most general LMI constraint definition, the activation functions were assumed to not necessarily be the ReLU function but a general element-wise activation function, which were $L$-smooth and $m$-strongly convex, where $L_i \ge m_i$. Thus, the general activation function quadratic constraint was used \cite{Araujo2023, Xu2024}:

\begin{small}
\begin{align}
\begin{bmatrix}
        v_k - v'_k \\
        w'_{k, i} - w_{k, i}
    \end{bmatrix}^\top  
    \begin{bmatrix}
        -2 L_i m_i \Lambda_i  & (m_i + L_i)\Lambda_i \\
         (m_i + L_i)\Lambda_i & -2 \Lambda_i
    \end{bmatrix}
    \begin{bmatrix}
        v_k - v'_k  \\
        w'_{k, i} - w_{k,i}
    \end{bmatrix} \leq 0 
\end{align}
\end{small}

where $\Lambda_n$ must be a positive definitive diagonal matrix. Given that $v_k - v'_k =  C_n \left(w_{k, n - 1}  - w'_{k, n - 1}\right)$ the inequality thus becomes the following quadratic constraints, where $\Delta w_{k, i}$ was defined as $\Delta w_{k, i} =  w'_{k, i} - w_{k,i}$,

\begin{tiny}
\begin{align}
    \begin{bmatrix}
        C_1 \left(x'_k - x_k \right) \\
        \Delta w_{k, 1}
    \end{bmatrix}^\top  &\begin{bmatrix}
        -2 L_1 m_1 \Lambda_1  & (m_1 + L_1)\Lambda_1 \\
         (m_1 + L_1)\Lambda_1 & -2 \Lambda_1
    \end{bmatrix}\begin{bmatrix}
        C_1 \left(x'_k - x_k \right) \\
        \Delta w_{k, 1}
    \end{bmatrix} \leq 0,  \nonumber \\
    \begin{bmatrix}
        C_2 \left(\Delta w_{k, 1} \right) \\
        \Delta w_{k, 2}
    \end{bmatrix}^\top  &\begin{bmatrix}
        -2 L_2 m_2 \Lambda_2  & (m_2 + L_2)\Lambda_2 \\
         (m_2 + L_2)\Lambda_2 & -2 \Lambda_2
    \end{bmatrix}\begin{bmatrix}
        C_2 \left(\Delta w_{k, 1} \right) \\
        \Delta w_{k, 2}
    \end{bmatrix} \leq 0, \nonumber \\
    &\vdots \nonumber \\
        \begin{bmatrix}
        C_n \left(\Delta w_{k, n-1} \right) \\
        \Delta w_{k, n}
    \end{bmatrix}^\top  &\begin{bmatrix}
        -2 L_n m_n \Lambda_n  & (m_n + L_n)\Lambda_n \\
         (m_n + L_n)\Lambda_n & -2 \Lambda_n
    \end{bmatrix}\begin{bmatrix}
        C_n \left(\Delta w_{k, n -1 } \right) \\
        \Delta w_{k, n}
    \end{bmatrix} \leq 0 .
\end{align}
\end{tiny}
To combine the LMIs, a concatenated vector of all the $w_{k,n}$ and $x_k$ was created to sum all the conditions and solve them all together. The following LMI could thus be formulated as the summation in Equation \eqref{eqn:lmi_formulation}.
\begin{table*}
\centering
\begin{align}
\begin{bmatrix}
    x'_k - x_k \\
    w'_{k, 1} - w_{k, 1}\\
    w'_{k, 2} - w_{k, 2}\\
    \vdots \\
    w'_{k, n - 1} - w_{k, n - 1} \\
    w'_{k, n} - w_{k, n}
\end{bmatrix}^\top  \begin{bmatrix}
    I_{d_x} & 0_{d_x}  \\
    \vdots & 0_{d_1} \\
    \vdots & \vdots \\
    \vdots & \vdots \\
    0_{d_{n - 1}} & \vdots \\
    0_{d_x} & I_{d_x} 
\end{bmatrix} \begin{bmatrix}
    A_k^\top A_k - \mathcal{L}^2I & A_k^\top  B_k \\
    B_k^\top A_k & B_k^\top  B_k
\end{bmatrix}\begin{bmatrix}
    I_{d_x} & 0_{d_x}  \\
    \vdots & 0_{d_1} \\
    \vdots & \vdots \\
    \vdots & \vdots \\
    0_{d_{n - 1}} & \vdots \\
    0_{d_x} & I_{d_x} 
\end{bmatrix}^\top 
\begin{bmatrix}
    x'_k - x_k \\
    w'_{k, 1} - w_{k, 1}\\
    w'_{k, 2} - w_{k, 2}\\
    \vdots \\
    w'_{k, n - 1} - w_{k, n - 1} \\
    w'_{k, n} - w_{k, n}
\end{bmatrix}  +  \nonumber \\
\sum_{l = 1}^n 
\begin{bmatrix}
    x'_k - x_k \\
    w'_{k, 1} - w_{k, 1}\\
    w'_{k, 2} - w_{k, 2}\\
    \vdots \\
    w'_{k, n - 1} - w_{k, n - 1} \\
    w'_{k, n} - w_{k, n}
\end{bmatrix}^\top  E_i^\top  \begin{bmatrix}
    C_l & 0 \\
    0 & I
\end{bmatrix}^\top   \begin{bmatrix}
        -2 L_l m_l \Lambda_l  & (m_l + L_l)\Lambda_l \\
         (m_l + L_l)\Lambda_l & -2 \Lambda_l
    \end{bmatrix}\begin{bmatrix}
    C_l & 0 \\
    0 & I
\end{bmatrix}
E_i
\begin{bmatrix}
    x'_k - x_k \\
    w'_{k, 1} - w_{k, 1}\\
    w'_{k, 2} - w_{k, 2}\\
    \vdots \\
    w'_{k, n - 1} - w_{k, n - 1} \\
    w'_{k, n} - w_{k, n}
\end{bmatrix} \leq 0, \label{eqn:lmi_formulation}
\end{align} 
\end{table*}
Where,
\begin{align}
    D_l &= \sum_{i = 1}^l d_i,  \\
    E_l &: \{0, 1\}^{\left(d_{l} + d_{l - 1}\right) \times D_n}, \\ [E_l]_{ij} &= \begin{cases}
    1 & \text{if $j - D_l = i$} \nonumber \\
    0 & \text{else}
\end{cases},
\end{align}
moreover, $i$ and $j$ represented the row and column, respectively. The $E_l$ matrix represented a "selection" vector to ensure that the proper variables were used for the parameterization. Which gave the following resultant LMI in Equation \eqref{eqn:explicit_lmi}.
\begin{table*}
\centering
\begin{align}
    \begin{bmatrix}
        A^\top  A - \mathcal{L}^2I - 2 L_1 m_1 C_{1}^\top  \Lambda_{1} C_{1} & (L_1 + m_1) C_{1}^\top  \Lambda_{1}                           & 0      & 0        & 0                                                         &  A^\top B                      \\
       (L_1 + m_1)   \Lambda_{1} C_{1}                              & - 2 L_2 m_2 C_{2}^\top  \Lambda_{2} C_{2} -2  \Lambda_{1}   & \ddots & 0        & 0                                                         &  0                         \\
        0                                                       & \ddots                                                & \ddots & \ddots   & 0                                                         &  0                         \\
        0                                                       & 0                                                     & \ddots & \ddots   & \ddots                                                    &  0                         \\
        0                                                       & 0                                                     & 0      & \ddots   &  - 2 L_n m_n C_{n}^\top  \Lambda_{n} C_{n} -2  \Lambda_{n - 1}  &  (L_n + m_n)  C_n^\top  \Lambda_n  \\
        B^\top  A                                                   & 0                                                     & 0      & 0        &(L_n + m_n)  \Lambda_n  C_n                                    &  B^\top B-2\Lambda_n         
    \end{bmatrix} \preceq 0, \label{eqn:explicit_lmi}
\end{align}
\end{table*}
The question then became what parameterization of $\{\Lambda_1, \cdots, \Lambda_n\}$,$\{C_1, \cdots, C_n\}$, and $B$ would be needed to ensure that the LMI was indeed negative semi-definitive to satisfy the Lipschitz constraint where ideally $\{C_1, \cdots, C_n\}$ would be as unconstrained as possible to ensure expressive inner layers. From the LMI, it could be noticed that it was exceedingly complex to derive the constraint of the network explicitly based on the eigenvalues of the network. As such, although it only provided loose bounds on the eigenvalues, the Gershgorin circle theorem could be used to derive bounds on the network.

\begin{theorem}
    Let $A$ be a complex matrix $n \times n$ matrix, with entries $a_{ij}$. For $i \in  \{1, \cdots,n\}$ let $R_i$ be the sum of the absolute values of the non-diagonal entries of the $i$-th row:
    \begin{align}
        R_i &= \sum_{j \neq i} \abs{a_{ij}}.
    \end{align}
    Let $D(a_{ii}, R_{i}) \subseteq \mathbb{C}$ be a closed disc centered at $a_{ii}$ with radius $R_i$, every eigenvalue of $A$ lies within at least one of the Gershgorin discs $D(a_{ii}, R_{i})$.
\end{theorem}

The following corollary was thus derived to generate conditions to ensure the LMI would be negative semi-definitive.

\begin{corollary}
    If all the Gershgorin discs of a matrix $A$ are defined in the negative real plane, $\mathbb{R}_{-}$, for $i \in  \{1, \cdots,n\}$ $\Re{(a_{ii} + R_i)} \leq 0$, then the matrix $A$ must be negative semi-definitive. 
\end{corollary}

The conditions necessary to ensure that the overall LMI matrix $M$ was negative semi-definite were derived by analyzing its Gershgorin discs. The analysis required demonstrating that all Gershgorin discs were entirely contained within the left half-plane, ensuring that the eigenvalues of $M$ were non-positive. Given the structure of the LMI, the matrix could be decomposed into three distinct sections: the first block, the middle blocks, and the last block. For each block, a corresponding set of constraints on the desired parameters was determined to ensure the feasibility of the problem. 

As the LMI matrix was symmetric, the Gershgorin discs derived from the rows were shown to coincide with those derived from the columns. This symmetry allowed the analysis to be carried out equivalently from either perspective without losing generality.

\section{General LMI Solution}

For notation, the parameters $S_a$ and $P_a$ were defined as $S_a = L_a + m_a$ and $P_a = L_a m_a$ to help reduce the notation size.

\subsection{Last block}

Below is the derivation of the constraints for the parameters defined in the last block portion of the LMI. 
\begin{theorem}
    For the parameter $C_n$, the norm of the rows must be upper bounded by,
    \begin{align}
        \norm{C_{n,i}}_1  < \frac{2}{\abs{L_n + m_n}},
    \end{align}
    while $\lambda_{n, i}$ must be lower bounded by,
\begin{align}
   \lambda_{n, i} &\ge \frac{b_i^2  + \abs{a_i} \abs{b_i}}{2 -  \abs{L_n + m_n}\norm{C_{n,i}}_1 }.
\end{align}
\end{theorem}
\begin{proof}
The final matrix row block was represented through the parameters where $l = n$:
\begin{align}
    \left\{
        B^\top  A,         (L_n + m_n) \Lambda_n  C_n,        B^\top B-2\Lambda_n 
    \right\}. \nonumber
\end{align}
Which gave the following Gershgorin discs for $\forall i \{1, \cdots, m_{n}\}$ (where $m_x = m_n$):
\begin{align}
    D\left(b_i^2 - 2\lambda_{n, i}, \abs{a_i} \abs{b_i} + \sum_{j = 1}^{d_{n- 1}}\abs{(L_n + m_n)c_{n, i, j} \lambda_{n, i}}\right),
\end{align}
For which the upper bound constraint was thus:
\begin{align}
    \epsilon_{3, n ,i, \max} &= b_i^2 - 2\lambda_{n, i} + \abs{a_i} \abs{b_i} + \abs{S_n}\sum_{j = 1}^{d_{n- 1}}\abs{ c_{n, ij} \lambda_{n, i} },
\end{align}
Applying the negative-semi definitive constraint, the following constraint was derived:
\begin{align}
   0 &\ge b_i^2 - 2\lambda_{n, i} + \abs{a_i} \abs{b_i} +\abs{L_n + m_n} \sum_{j = 1}^{d_{n- 1}}\abs{c_{n, i, j} \lambda_{n, i}}, \nonumber\\
   &\ge b_i^2  + \abs{a_i} \abs{b_i} + \lambda_{n, i} \left( \abs{L_n + m_n} \sum_{j = 1}^{d_{n- 1}}\abs{c_{n, ij} } - 2 \right), \nonumber \\
   \lambda_{n, i} &\ge \frac{b_i^2  + \abs{a_i} \abs{b_i}}{2 -  \abs{L_n + m_n}\sum_{j = 1}^{d_{n- 1}}\abs{c_{n, ij} } }.
\end{align}
Given that all $\lambda_{n, i}$ must be positive definitive, the only way to ensure this was to ensure that:
\begin{align}
    \sum_{j = 1}^{d_{n- 1}}\abs{c_{n, ij} } = \norm{C_{n,i}}_1  &< \frac{2}{\abs{L_n + m_n}}.
\end{align}
Thus enforcing that all the rows of $C_n$ must be strictly upper bound by $\frac{2}{\abs{L_n + m_n}}$.
\end{proof}
\subsection{Middle blocks}

Below is the derivation of the constraints for the parameters defined in the middle blocks of the LMI. 
\begin{theorem}
    For all $l = {1, \cdots, n - 1}$ the parameter $C_l$ must have its row norm be upper bounded by,
    \begin{align}
        \norm{C_{n,i}}_1  < \frac{2}{\abs{L_n + m_n}},
    \end{align}
    and element-wise upper bounded by 
    \begin{align}
         \abs{c_{l + 1, ji}} \leq &  \frac{\abs{S_{l + 1}} ^2+4  \abs{P_{l + 1}} }{2 (\abs{ P_{l + 1}} +P_{l + 1}) \abs{ S_{l + 1}} }.
    \end{align}
    while $\lambda_{l, i}$ must be lower bounded by,
    \begin{align}
       \lambda_{l,i } \ge& \frac{\sum^{d_{l + 1}}_{j = 1} \lambda_{l + 1, j} \left( \abs{S_{l+1}} \abs{ c_{l + 1,j i}} -2 P_{l + 1} c_{l + 1, ji}^2 \right)}{2 - \abs{S_{l}} \sum^{d_{l-1}}_{j = 1} \abs{c_{l,i j}} }  
       \nonumber \\ &+ \frac{2\abs{P_{l + 1}} \sum_{z= 1, z \neq i}^{d_{l}} \abs{\sum_{j = 1}^{d_{l + 1}} \lambda_{l + 1, j} c_{l + 1, ji} c_{l + 1, jz}  }}{2 - \abs{S_{l}} \sum^{d_{l-1}}_{j = 1} \abs{c_{l,i j}} }.
    \end{align}
\end{theorem}
\begin{proof}
The set of block matrices that represent the middle block represented the parameters where $\forall l \{1, \cdots, n - 1\}$ were:
\begin{align}
    \left\{
       S_l   \Lambda_{l} C_{l},   - 2 P_{l + 1} C_{l + 1}^\top  \Lambda_{l + 1} C_{l + 1} -2  \Lambda_{l}  ,  S_{l+1}  C_{l + 1}^\top  \Lambda_n 
    \right\}. \nonumber
\end{align}
Which gave the following Gershgorin discs, $D\left(a_{ii},  R_{i }\right)$, for $\forall i \{1 \cdots m_{l}\} \forall l \{1 \cdots n - 1\} $:
\begin{align}
    a_{ii} =& -2(L_{l + 1} m_{l + 1})\sum_{j=1}^{d_{l + 1}}\lambda_{l + 1, j} c_{l + 1, ji}^2-2 \lambda_{l, i}, \\
    R_{i}  =& \lambda_{l, i} \abs{L_{l} + m_{l}}  \sum^{d_{l-1}}_{j = 1} \abs{c_{l,i j}} \nonumber \\ 
     &+ \abs{L_{l + 1} + m_{l + 1}}\sum^{d_{l + 1}}_{j = 1} \lambda_{l + 1, j} \abs{c_{l + 1,j i}} \nonumber \\ 
     &+  2\abs{L_{l + 1} m_{l  + 1}} \sum_{z= 1, z \neq i}^{d_{l}} \abs{\sum_{j = 1}^{d_{l + 1}} \lambda_{l + 1, j} c_{l + 1, ji} c_{l + 1, jz}  } .
\end{align}
For which the upper bound constraint was thus:
\begin{align}
    \epsilon_{2, l, i, \max} =& a_{ii} + R_i, \\
    =& \lambda_{l, i} \left(\abs{S_{l}} \sum^{d_{l-1}}_{j = 1} \abs{c_{l,i j}} -2 \right) \nonumber \\ & +  \sum^{d_{l + 1}}_{j = 1} \lambda_{l + 1, j} \left( \abs{S_{l+1}} \abs{ c_{l + 1,j i}} -2 P_{l + 1} c_{l + 1, ji}^2 \right) \nonumber \\
    &+  2\abs{P_{l + 1}} \sum_{z= 1, z \neq i}^{d_{l}} \abs{\sum_{j = 1}^{d_{l + 1}} \lambda_{l + 1, j} c_{l + 1, ji} c_{l + 1, jz}  }.
\end{align}
Applying the negative-semi definitive constraint, the following constraint was derived:
\begin{align}
   \lambda_{l,i } \ge& \frac{\sum^{d_{l + 1}}_{j = 1} \lambda_{l + 1, j} \left( \abs{S_{l+1}} \abs{ c_{l + 1,j i}} -2 P_{l + 1} c_{l + 1, ji}^2 \right)}{2 - \abs{S_{l}} \sum^{d_{l-1}}_{j = 1} \abs{c_{l,i j}} }  
   \nonumber \\ &+ \frac{2\abs{P_{l + 1}} \sum_{z= 1, z \neq i}^{d_{l}} \abs{\sum_{j = 1}^{d_{l + 1}} \lambda_{l + 1, j} c_{l + 1, ji} c_{l + 1, jz}  }}{2 - \abs{S_{l}} \sum^{d_{l-1}}_{j = 1} \abs{c_{l,i j}} },
\end{align}
Given that all $\lambda$ must be positive definitive and that the numerator and denominator are independent, the following constraints could thus be derived:
\begin{align}
    \sum_{j = 1}^{d_{n- 1}}\abs{c_{l, ij} } = \norm{C_{l,i}}_1  &< \frac{2}{\abs{L_l + m_l}}.
\end{align}
Thus enforcing that all the rows of $C_l$ had to be strictly upper bound by $\frac{2}{\abs{L_l + m_l}}$.

In addition, the numerator was analyzed to ensure that the system remained positive and definite. A simplified approach was adopted by imposing the following conditions:
\begin{align}
     \abs{S_{l+1}} \abs{ c_{l + 1,j i}} &\ge 2 P_{l + 1} c_{l + 1, ji}^2,  \nonumber\\
     \abs{S_{l+1}} &\ge 2 P_{l + 1} \abs{c_{l + 1, ji}},  \nonumber \\
     \abs{c_{l + 1, ji}}  &\leq \frac{\abs{S_{l+1}}}{2 P_{l + 1} }.
\end{align}
The off-diagonal terms were examined to derive a less restrictive upper bound for the variable $C_l$. The radius $R_i$ was increased to produce a more conservative estimate, which, although broader, facilitated the inclusion of the off-diagonal terms within the inequality framework.
\begin{align}
    =& \sum^{d_{l + 1}}_{j = 1} \lambda_{l + 1, j} \left( \abs{S_{l+1}} \abs{ c_{l + 1,j i}} -2 P_{l + 1} c_{l + 1, ji}^2 \right)   \nonumber \\ & + 2\abs{P_{l + 1}} \sum_{z= 1, z \neq i}^{d_{l}} \abs{\sum_{j = 1}^{d_{l + 1}} \lambda_{l + 1, j} c_{l + 1, ji} c_{l + 1, jz}  },  \nonumber \\
    \leq& \sum^{d_{l + 1}}_{j = 1} \lambda_{l + 1, j} \left( \abs{S_{l+1}} \abs{ c_{l + 1,j i}} -2 P_{l + 1} c_{l + 1, ji}^2 \right)  \nonumber \\ & + 2\abs{P_{l + 1}} \sum_{z= 1, z \neq i}^{d_{l}} \sum_{j = 1}^{d_{l + 1}} \lambda_{l + 1, j}\abs{ c_{l + 1, ji} c_{l + 1, jz}  }, \nonumber \\
    =& \sum^{d_{l + 1}}_{j = 1} \lambda_{l + 1, j} \left( \abs{S_{l+1}} \abs{ c_{l + 1,j i}} -2 P_{l + 1} c_{l + 1, ji}^2 \right)  \nonumber \\ & + 2\abs{P_{l + 1}} \sum_{j = 1}^{d_{l + 1}}\lambda_{l + 1, j}\abs{ c_{l + 1, ji}} \sum_{z= 1, z \neq i}^{d_{l}} 
  \abs{ c_{l + 1, jz}  }, \nonumber \\
  =& \sum^{d_{l + 1}}_{j = 1} \lambda_{l + 1, j} \left( \abs{S_{l+1}} \abs{ c_{l + 1,j i}} -2 P_{l + 1} c_{l + 1, ji}^2 \right)  \nonumber \\ & + 2\abs{P_{l + 1}} \sum_{j = 1}^{d_{l + 1}}\lambda_{l + 1, j}\abs{ c_{l + 1, ji}} \left( \sum_{z= 1}^{d_{l}} 
  \abs{ c_{l + 1, jz}  } - \abs{ c_{l + 1, ji}  } \right),  \nonumber \\
  \leq & \sum^{d_{l + 1}}_{j = 1} \lambda_{l + 1, j} \left( \abs{S_{l+1}} \abs{ c_{l + 1,j i}} -2 P_{l + 1} c_{l + 1, ji}^2 \right)  \nonumber \\ &\qquad + 2\abs{P_{l + 1}} \sum_{j = 1}^{d_{l + 1}}\lambda_{l + 1, j}\abs{ c_{l + 1, ji}} \left( \frac{2}{\abs{S_{l+1}}} - \abs{ c_{l + 1, ji}  } \right), \nonumber\\
  = & \sum^{d_{l + 1}}_{j = 1} \lambda_{l + 1, j} \Bigg[ \abs{S_{l+1}} \abs{ c_{l + 1,j i}} -2 P_{l + 1} c_{l + 1, ji}^2  \nonumber \\ &\qquad\qquad  + 2\abs{P_{l + 1}} \abs{ c_{l + 1, ji}} \left( \frac{2}{\abs{S_{l+1}}} - \abs{ c_{l + 1, ji}  } \right) \Bigg].
\end{align}
Where the inner term needed to be constrained:
\begin{align}
    0 \leq &\abs{S_{l+1}} \abs{ c_{l + 1,j i}} -2 P_{l + 1} c_{l + 1, ji}^2 \nonumber \\ & + 2\abs{P_{l + 1}} \abs{ c_{l + 1, ji}} \left( \frac{2}{\abs{S_{l+1}}} - \abs{ c_{l + 1, ji}  } \right), \nonumber \\
    =& \abs{S_{l+1}} \abs{ c_{l + 1,j i}} -2 P_{l + 1} c_{l + 1, ji}^2 \nonumber \\ & \qquad\qquad + 2\abs{P_{l + 1}}  \left( \frac{2 \abs{ c_{l + 1, ji}}}{\abs{S_{l+1} }} -  c_{l + 1, ji}^2 \right) , \nonumber \\
    =& \abs{S_{l+1}} \abs{ c_{l + 1,j i}} +   \frac{4 \abs{P_{l + 1}} }{\abs{S_{l+1}}}\abs{ c_{l + 1, ji}} \nonumber \\ & \qquad\qquad  -2 \left(P_{l + 1} + \abs{P_{l + 1}} \right) c_{l + 1, ji}^2, \nonumber \\
    0 \leq & \abs{S_{l+1}}  +   \frac{4 \abs{P_{l + 1}} }{\abs{S_{l+1}}} -2 \left(P_{l + 1} + \abs{P_{l + 1}} \right) \abs{c_{l + 1, ji}} \nonumber \\
    \abs{c_{l + 1, ji}} \leq &  \frac{\abs{S_{l + 1}} ^2+4  \abs{P_{l + 1}} }{2 (\abs{ P_{l + 1}} +P_{l + 1}) \abs{ S_{l + 1}} }.
\end{align}
Given the additional element-wise constraint, it was observed that there were two situations in which the constraint became irrelevant. The first scenario occurred when $ L_{l+1} m_{l+1} \leq 0 $, and the second arose when $ L_{l+1} + m_{l+1} = 0 $. This condition was satisfied, for instance, in the case of a ReLU activation function, where $ L = 1 $ and $ m = 0 $.
\end{proof}
\subsection{First layer}

Finally, below is the derivation of the constraints for the parameters defined in the first-row block of the LMI. 
\begin{theorem}
    The parameter $C_1$ must be element-wise upper bounded by 
    \begin{align}
         \abs{c_{ 1, ji}} \leq &  \frac{\left(\mathcal{L}^2 - a_i^2  - \abs{a_i} \abs{b_i}\right)\abs{S_{1}}}{d_1 \lambda_{1, j} \left(\abs{S_{1}}^2 +  4\abs{P_{1}}\right)},
    \end{align}
\end{theorem}

\begin{proof}
This block represented the parameters where $l = 1$:
\begin{align}
    \left\{
        A^\top  A - \mathcal{L}^2I - 2 L_1 m_1 C_{1}^\top  \Lambda_{1} C_{1} , (L_1 + m_1) C_{1}^\top  \Lambda_{1} ,  A^\top B 
    \right\}. \nonumber
\end{align}
Which gave the following Gershgorin discs for $\forall i \{1 \cdots m_{x}\}$ (where $m_x = m_n$):
\begin{align}
    a_{ii} =& a_i^2 -\mathcal{L}^2 -2 (L_{1} m_{1})\sum_{j = 1}^{d_1} \lambda_{1, j} c_{1, ji}^2, \\
    R_{i}  =&  \abs{a_i} \abs{b_i} + \abs{L_1 + m_1}\sum_{j = 1}^{d_1}\lambda_{1, j} \abs{c_{1, ji}} \nonumber \\ &+ 2\abs{L_{1} m_{ 1}} \sum_{z= 1, z \neq i}^{d_{x}} \abs{\sum_{j = 1}^{d_{1}} \lambda_{1, j} c_{1, ji} c_{1, jz}  }.
\end{align}
For which the upper bound constraint was thus:
\begin{align}
    \epsilon_{1, 1 ,i, \max} =& a_i^2 -\mathcal{L}^2 -2 P_{1}\sum_{j = 1}^{d_1} \lambda_{1, j} c_{1, ji}^2 +  \abs{a_i} \abs{b_i} \nonumber \\ & + \abs{S_{1}}\sum_{j = 1}^{d_1}\lambda_{1, j} \abs{c_{1, ji}} \nonumber \\ &+ 2\abs{P_{1}} \sum_{z= 1, z \neq i}^{d_{x}} \abs{\sum_{j = 1}^{d_{1}} \lambda_{1, j} c_{1, ji} c_{1, jz}  } .
\end{align}
Applying the negative-semi definitive constraint, the following constraint was derived:
\begin{align}
    0 \ge&    a_i^2 -\mathcal{L}^2 -2 P_{1}\sum_{j = 1}^{d_1} \lambda_{1, j} c_{1, ji}^2 +  \abs{a_i} \abs{b_i} + \abs{S_{1}}\sum_{j = 1}^{d_1}\lambda_{1, j} \abs{c_{1, ji}}  \nonumber \\ & + 2\abs{P_{1}} \sum_{z= 1, z \neq i}^{d_{x}} \abs{\sum_{j = 1}^{d_{1}} \lambda_{1, j} c_{1, ji} c_{1, jz}  } ,\nonumber \\
    \leq&    a_i^2 -\mathcal{L}^2 -2 P_{1}\sum_{j = 1}^{d_1} \lambda_{1, j} c_{1, ji}^2 +  \abs{a_i} \abs{b_i} + \abs{S_{1}}\sum_{j = 1}^{d_1}\lambda_{1, j} \abs{c_{1, ji}}  \nonumber \\ &+ 2\abs{P_{1}} \sum_{z= 1, z \neq i}^{d_{x}} \sum_{j = 1}^{d_{1}} \lambda_{1, j} \abs{c_{1, ji} c_{1, jz}},  \nonumber \\
    =&    a_i^2 -\mathcal{L}^2 -2 P_{1}\sum_{j = 1}^{d_1} \lambda_{1, j} c_{1, ji}^2 +  \abs{a_i} \abs{b_i} + \abs{S_{1}}\sum_{j = 1}^{d_1}\lambda_{1, j} \abs{c_{1, ji}} \nonumber \\ & + 2\abs{P_{1}}  \sum_{j = 1}^{d_{1}}  \lambda_{1, j} \abs{c_{1, ji}} \sum_{z= 1, z \neq i}^{d_{x}}\abs{c_{1, jz}},  \nonumber \\
    =&    a_i^2 -\mathcal{L}^2 +  \abs{a_i} \abs{b_i}  +\sum_{j = 1}^{d_1} \lambda_{1, j} \Bigg(\abs{S_{1}} \abs{c_{1, ji}}-2 P_{1}  c_{1, ji}^2 \nonumber \\ & \qquad\qquad\qquad\qquad\qquad  + 2\abs{P_{1}} \abs{c_{1, ji}} \sum_{z= 1, z \neq i}^{d_{x}}\abs{c_{1, jz}} \Bigg),  \nonumber \\
    \leq&    a_i^2 -\mathcal{L}^2 +  \abs{a_i} \abs{b_i}  +\sum_{j = 1}^{d_1} \lambda_{1, j} \Bigg[ \abs{S_{1}} \abs{c_{1, ji}}-2 P_{1}  c_{1, ji}^2 \nonumber \\ & \qquad\qquad\qquad\qquad   + 2\abs{P_{1}} \abs{c_{1, ji}}\left( \frac{2}{\abs{S_{1}}} - \abs{ c_{l + 1, ji}  } \right) \Bigg], \nonumber \\
    =&   \sum_{j = 1}^{d_1} \Bigg[ \frac{a_i^2 -\mathcal{L}^2 +  \abs{a_i} \abs{b_i}}{d_1}  + \lambda_{1, j} \Bigg(\abs{S_{1}} \abs{c_{1, ji}}-2 P_{1}  c_{1, ji}^2 \nonumber \\ & \qquad\qquad\qquad\qquad   + 2\abs{P_{1}} \abs{c_{1, ji}}\left( \frac{2}{\abs{S_{1}}} - \abs{ c_{1, ji}  } \right)  \Bigg) \Bigg] .
\end{align}
The constraint $ L_1 m_1 \leq 0 $ was enforced to ensure the solvability of the system. Consequently, the system of equations was formulated as follows:
\begin{align}
    0 \ge & \frac{a_i^2 -\mathcal{L}^2 +  \abs{a_i} \abs{b_i}}{d_1}  \nonumber \\ &+ \lambda_{1, j} \Bigg(\abs{S_{1}} \abs{c_{1, ji}}-2 P_{1}  c_{1, ji}^2  \nonumber \\ &\qquad\qquad+ 2\abs{P_{1}} \abs{c_{1, ji}}\left( \frac{2}{\abs{S_{1}}} - \abs{ c_{1, ji}  } \right)  \Bigg), \nonumber  \\
    = & \frac{a_i^2 -\mathcal{L}^2 +  \abs{a_i} \abs{b_i}}{d_1} + \lambda_{1, j} \left(\abs{S_{1}} \abs{c_{1, ji}} +  \frac{4\abs{P_{1}}}{\abs{S_{1}}} \abs{c_{1, ji}} \right),   \nonumber \\
     \abs{c_{1, ji}}  \leq& \frac{\mathcal{L}^2 - a_i^2  -   \abs{a_i} \abs{b_i}}{d_1 \lambda_{1, j}} \frac{1}{\abs{S_{1}} +  \frac{4\abs{P_{1}}}{\abs{S_{1}}} }, \nonumber \\
     =& \frac{\left(\mathcal{L}^2 - a_i^2  - \abs{a_i} \abs{b_i}\right)\abs{S_{1}}}{d_1 \lambda_{1, j} \left(\abs{S_{1}}^2 +  4\abs{P_{1}}\right)}.
\end{align}
Which then induced the inequality that $\mathcal{L}^2 - a_i^2 - \abs{a_i} \abs{b_i} \ge 0$, which in turn gave the consrtaints that $a_i \in (-\mathcal{L}, \mathcal{L})$ with $\abs{b_{i}} < \frac{\mathcal{L}^2 - a_i^2}{\abs{a_i}}$. This thus completed the LMI constraints.
\end{proof}

Given all the derived constraints, the complete set of constraints of the neural network was listed in Table \ref{tab:matrix_constrains}.

\begin{table*}[htb]
\centering
\caption{General LMI Condensed Constraints}
\label{tab:elementwise_constraints}
\begin{tabular}{|Sc|Sc|Sc|Sc|l}
\cline{1-4}
Parameter            & Inequality & Constraint                                                                                                         & Indexing                                                        \\ \cline{1-4}
 $S_l$               & $=$    & $L_l + m_l$                                                                                                            & $\forall l \{1, \cdots, n\}$                                    \\ \cline{1-4}
 $P_l$               & $=$    & $L_l m_l$                                                                                                              & $\forall l \{1, \cdots, n\}$                                    \\ \cline{1-4}
 $\lambda_{n, i}$    & $\ge$  & $\frac{b_i^2  + a_i b_i}{2 -  \abs{S_n}\sum_{j = 1}^{d_{n- 1}}\abs{c_{n, ij} } }$                                      & $\forall i \{1, \cdots, d_{n}\}$                                \\ \cline{1-4}
 $\norm{C_{n,i}}_1$  & $< $   & $\frac{2}{\abs{S_l}}$                                                                                                  & $\forall i \{1, \cdots, d_{l}\} \forall l \{1, \cdots, n\}$     \\ \cline{1-4}
 $ \lambda_{l,i }$   & $\ge$  & $\frac{\sum^{d_{l + 1}}_{j = 1} \lambda_{l + 1, j} \left( \abs{S_{l+1}} \abs{ c_{l + 1,j i}} -2 P_{l + 1} c_{l + 1, ji}^2 \right) + 2\abs{P_{l + 1}} \sum_{z= 1, z \neq i}^{d_{l}} \abs{\sum_{j = 1}^{d_{l + 1}} \lambda_{l + 1, j} c_{l + 1, ji} c_{l + 1, jz}  }}{2 - \abs{S_{l}} \sum^{d_{l-1}}_{j = 1} \abs{c_{l,i j}} }  $
                                                                                                                                                       & $\forall i \{1, \cdots, d_{l}\} \forall l \{1, \cdots, n - 1\}$ \\ \cline{1-4}
 $\abs{c_{l, ij}}$   & $\leq$ & $ \frac{\abs{S_{l}} ^2+4  \abs{P_{l}} }{2 (\abs{ P_{l}} +P_{l }) \abs{S_{l}} }$                                        & $\forall i \{1, \cdots, d_{l}\} \forall j \{1, \cdots, d_{l - 1}\} \forall l \{2, \cdots, n\} $    \\ \cline{1-4}
 $\abs{c_{1, ij}} $  & $\leq$ & $\frac{\left(\mathcal{L}^2 - a_i^2  - a_i b_i\right)\abs{S_{1}}}{d_1 \lambda_{1, i} \left(\abs{S_{1}}^2 +  4\abs{P_{1}}\right)}$   & $\forall i \{1, \cdots, d_{x}\}$                                \\ \cline{1-4}
 $\abs{a_i}$               & $\in$  & $(0, \mathcal{L})$                                                                                                               & $\forall i \{1, \cdots, d_{x}\}$                                \\ \cline{1-4}
 $\abs{b_i}$               & $\in$  & $[0, \frac{\mathcal{L}^2 - a_i^2}{\abs{a_i}})$                                                                                           & $\forall i \{1, \cdots, d_{x}\}$                                \\ \cline{1-4}
\end{tabular}
\end{table*}

The generalized versions of the equations could additionally be presented in matrix forms in Table \ref{tab:matrix_constrains}, where the absolute value function is applied element-wise, i.e., $\abs{A} = \{\abs{a_{ij}}\}$, for simplified computation and practicality, the diagonal matrices $\Lambda_l, A, B$ are represented as column vectors where the elements are the diagonal values.

\begin{table*}[htb]
\centering
\caption{General LMI Matrix Condensed Constraints}
\label{tab:matrix_constrains}
\begin{tabular}{|Sc|Sc|Sc|Sc|l}
\cline{1-4}
Parameter            & Inequality & Constraint                                                                                                         & Indexing                                   \\ \cline{1-4}
 $S_l$               & $=$    & $L_l + m_l$                                                                                                            &                                            \\ \cline{1-4}
 $P_l$               & $=$    & $L_l m_l$                                                                                                              &                                            \\ \cline{1-4}
 $D_l$               & $=$    & $\begin{bmatrix}\norm{C_{l,1}}_1 & \cdots & \norm{C_{l,m_{l}}}_1 \end{bmatrix}^\top  $                                     &    $\forall l \{1, \cdots, n\}$                                        \\ \cline{1-4}
 $\Lambda_{n}$       & $\ge$  & $\frac{B^2 + \abs{A}\abs{B}}{2 - \abs{S_n} D_n}$                                                                                   &                                            \\ \cline{1-4}
 $D_l$               & $< $   & $\frac{2}{\abs{S_l}}$                                                                                                  & $\forall l \{1, \cdots, n\}$               \\ \cline{1-4}
 $Q_l$               & $= $   & $C_{l}^\top  \text{diag}(\Lambda_{l}) C_{l}$                                                                               &                                            \\ \cline{1-4}
 $ \Lambda_{l}$      & $\ge$  & $\frac{\Lambda_{l + 1}^\top  (\abs{S_{l + 1}} \abs{C_{l + 1}} -2 P_{l + 1} C_{l + 1}^{\circ 2}) + 2 \abs{P_{l + 1}} \boldsymbol{1}^\top \abs{Q_{l + 1} -   \text{diag}( \text{diag}(Q_{l + 1})) }}{2 - \abs{S_l} D_l}$
                                                                                                                                                       & $\forall l \{1, \cdots, n - 1\}$           \\ \cline{1-4}
 $\abs{C_{l}}$       & $\leq$ & $\frac{\abs{S_{l}} ^2+4  \abs{P_{l}} }{2 (\abs{ P_{l}} +P_{l }) \abs{S_{l}} }$                                         & $\forall l \{2, \cdots, n\} $              \\ \cline{1-4}
 $\abs{C_{1}} $      & $\leq$ & $\frac{\abs{S_{1}}\left(\mathcal{L}^2 - A^2  - \abs{A}\abs{B}\right)}{d_1  \left(\abs{S_{1}}^2 +  4\abs{P_{1}}\right)} \Lambda_{1}^{-1}$                &                                            \\ \cline{1-4}
 $\abs{A}$                 & $\in$  & $(0, \mathcal{L})$                                                                                                               &                                            \\ \cline{1-4}
 $\abs{B}$                 & $\in$  & $[0, \frac{\mathcal{L}^2}{\abs{A}} - \abs{A})$                                                                                                 &                                            \\ \cline{1-4}
\end{tabular}
\end{table*}

\subsection{Weighted norm constraint} \label{sec:weighted}

For a weighted $\ell_1$-norm, it was desired to derive an unparameterized optimization formulation scheme for $x_i$ while ensuring the system remained upper bounded. Where $v_i > 0, \forall i$.
\begin{align}
    a \leq \norm{x}_{1, v} = \sum_{i = 1}^n v_i \abs{x_i}
\end{align}
The reparameterization $\abs{x_i} \leq  \pm \frac{a}{n} \frac{1}{v_i}$ was introduced, such that:
\begin{align}
    a  \leq \sum_{i = 1}^n v_i \abs{x_i} 
     \leq \sum_{i = 1}^n v_i \abs{\frac{a}{n} \frac{1}{v_i}} 
     = \frac{a}{n} \sum_{i = 1}^n1 
     = a
\end{align}
Where the constraint,
\begin{align}
    x_i = \frac{a}{n} \frac{1}{v_i} p_i, \label{eqn:weighted_constraint}
\end{align}
where $p_i \in (-1, 1)$. As such, to parameterize $x_i$, the optimization parameter for the network becomes optimizing $p_i$, where $p_i = \tanh(w_i)$, where $w_i$ was an unconstrained optimization parameter. As demonstrated by the normalization factor $\frac{\partial x_i}{\partial p_i}\propto O(\frac{1}{n v_i})$, which implied that the gradients of $x_i$ became proportionally smaller as the dimension of the vector became smaller. Small or vanishing gradients could cause problems for large and deeper networks.

\subsection{Elementwise vs. row constraint bound switching}

\subsubsection{$C_l$ constraints}

For the matrices $ C_l $ with $ l \in \{2, \cdots, n\} $, two simultaneous constraints were imposed on the system: the row-wise and element-wise constraints. In this context, the objective was to derive the upper bounds for the values of $ C_l $ that the optimization would be based on. 

The following constraint was derived from the norm constraints established in the unparameterized optimization formulation given by Equation \eqref{eqn:weighted_constraint}, where the upper bound was set as $ a = \frac{2}{\lvert S_l \rvert} $ and $ v_i = 1 $. The goal, therefore, was to identify the conditions under which the row-wise element constraint would dominate over the overall element-wise constraint.
\begin{align}
    \frac{2}{\abs{S_{l}} d_{l - 1}} &\leq \frac{\abs{S_{l}} ^2+4  \abs{P_{l}} }{2 (\abs{ P_{l}} +P_{l }) \abs{S_{l}} }, \nonumber\\
    d_{l - 1} &\ge \frac{4 (\abs{ P_{l}} +P_{l })}{\abs{S_{l}} ^2+4  \abs{P_{l}} }.
\end{align}
This thus informed us that when $P_l \leq 0$ ($\forall l, d_{l} \ge 0$), the element-wise constraint would always be greater than the element-wise, and if $P_l > 0$ and,
\begin{align}
       d_{l - 1} &\ge \frac{8 P_{l}}{\abs{S_{l}} ^2+4 P_{l}}.
\end{align}
By examining the maximum value of the bound, it was found that, due to the equation's symmetry concerning $ L_l $ and $ m_l $, solving for either the optimal value of $ m_l $ or $ L_l $ led to the optimal solution. This symmetry implied that both parameters contributed equivalently to the system, and thus, optimizing one in isolation was sufficient to determine the overall optimal configuration.
\begin{align}
    \frac{\partial}{\partial m_l}\frac{8 L_l m_l}{(L_l + m_l) ^2+4 L_l m_l} &= \frac{8 L_l (L_l-m_l) (L_l+m_l)}{\left(L_l^2+6 L_l m_l+m_l^2\right)^2},
\end{align}
solving for $0$ the optimal value was obtained when $m_l = \{-L_l, L_l\}$, where only the $m_l = L_l$ solution was kept due to the $P_l > 0$ constraint. Which gave the solution that (when $m_l = L_l$, $S_l = 2L, P_l = L_l^2$):
\begin{align}
    \frac{2}{\abs{S_{l}} d_{l - 1}} &\ge \frac{\abs{S_{l}} ^2+4  \abs{P_{l}} }{2 (\abs{ P_{l}} +P_{l }) \abs{S_{l}} }, \nonumber\\
    \frac{1}{\abs{L_l} d_{l - 1}} &\ge \frac{4 L_l ^2+4  L_l^2 }{8L_l^2 \abs{L_l} }, \nonumber\\
    \frac{1}{\abs{L_l} d_{l - 1}} &\ge \frac{1}{ \abs{L_l}}.
\end{align}
This demonstrated that even in the specific condition when $m_l = L_l$ and $d_l \leq 1$, the element and row-wise constraints would be equivalent to each other. This implied that the row-wise constraint would always be smaller than the element-wise constraint and should thus be the only one considered when constraining $C_l$ for $\forall \{2, \cdots, n\}$.
\subsubsection{$C_1$ constraints}

Upon analyzing the constraints derived for $C_1$, it was observed that a mutual dependence existed between $C_1$ and $\Lambda_1$. Specifically, the definition of $C_1$ necessitated the prior specification of $\Lambda_1$, and conversely, the determination of $\Lambda_1$ was contingent upon the specification of $C_1$. This interdependence introduced significant complexity in deriving an appropriate parameterization for $C_1$. As a result, an additional constraint was imposed on $C_1$ to address this issue, such that:
\begin{align}
    \abs{S_l} \sum_{j = 1}^{d_x} \abs{c_{1, ij}} \leq 1,
\end{align}
Which thus enforced the constraint that,
\begin{align}
G_l =& \Lambda_{l}^\top  (\abs{S_{l}} \abs{C_{l}} -2 P_{l} C_{l}^{\circ 2}) \nonumber \\ &+ 2 \abs{P_{l}} \boldsymbol{1}^\top \abs{Q_{l} -   \text{diag}( \text{diag}(Q_{l})) }, \\
    \lambda_{1,i } \ge&  \frac{G_{2}}{2 - \abs{S_{l}} \sum^{d_{x}}_{j = 1} \abs{c_{1,i j}} } \ge  G_{2},
\end{align}
where, $G_l$ represented the numerator of the $\Lambda_l$ parameterization. Enforcing this additional constraint on the row norm of $C_1$ thus imposed an upper bound of $\Lambda_1$, which no longer contained a dependence on $C_1$, breaking the cyclical parameterization. 

\subsection{LMI parameterization}

The eigenvalue distribution of the LMI was displayed below in Figure \ref{fig:eigenvalue_distribution} (The eigenvalue range was truncated to 10 times the quartile range; however, some of the eigenvalues have reached a magnitude of $-10^{11}$). To generate this distribution, all the parameters, the weights $C_l$ ( parameterized by $p_l$), and the biases $b_l$ were initialized with a uniform distribution. The biases and weights were initialized using the standard Kaiming initialization scheme, where the weights used tanh gains (i.e., scaling of 1 for tanh \cite{Kumar2017}) given that the variables $p_l$ were constrained by tanh.

\begin{figure}
    \centering
    \includegraphics[width=1\linewidth]{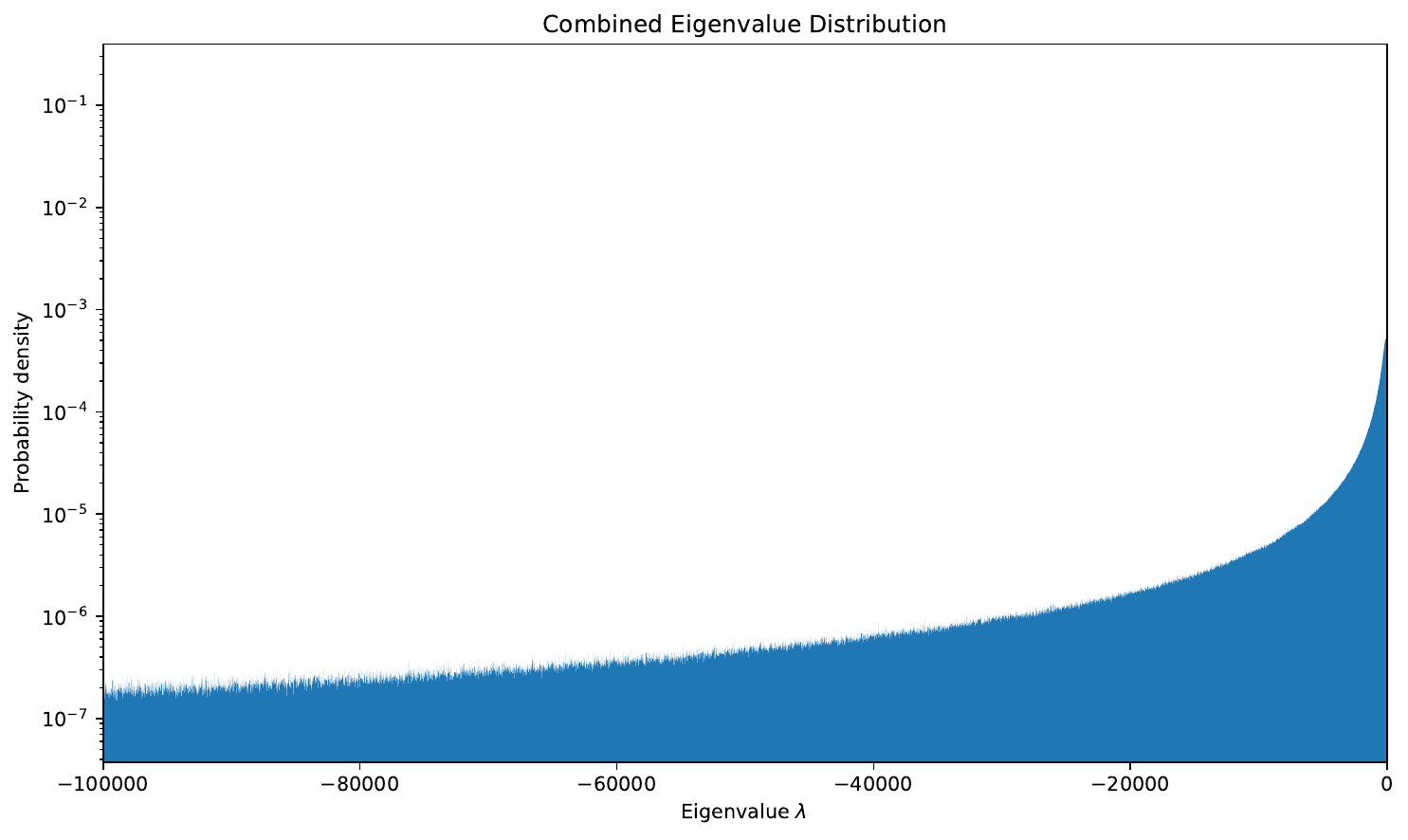}
    \caption{Eigenvalue distribution}
    \label{fig:eigenvalue_distribution}
\end{figure}

The constraints above, when implemented, thus generated the following example of Gershgorin circles for the LMI illustrated in Figure \ref{fig:lmi_gershgorin_circles}. The Figure demonstrates that the Gershgorin circles were all constrained on the negative real plane. 

\begin{figure}
    \centering
    \includegraphics[width=1\linewidth]{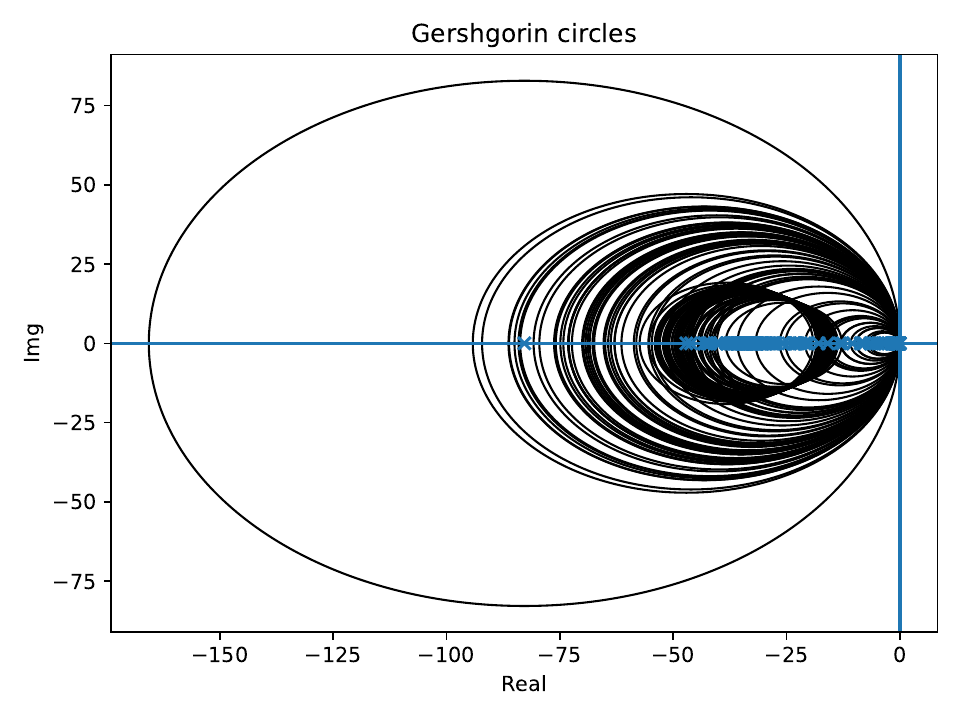}
    \caption{LMI Gershgorin Circles}
    \label{fig:lmi_gershgorin_circles}
\end{figure}

It was also interesting to observe that due to the recursive nature of the $\Lambda_l$ parameterization, the Gershgorin circles ended up encapsulating each other most of the time (this is not a general statement).

For the sake of completeness, the $L$ and $m$ constants of the activation functions defined in \href{https://pytorch.org/docs/stable/nn.html#non-linear-activations-weighted-sum-nonlinearity}{PyTorch} (assuming default values if not specified) were derived and defined in Table \ref{tab:activation_function_convecities}. It should be noted that the Hardshrink and RReLU could not be used due to their infinite $L, m$ constants; Hardshink has infinite $L, m$ due to its noncontinuous piece-wise definition, and PReLU due to its stochastic definition, which no longer made it's $L, m$ computable. Where,
\begin{align}
    \text{erfc}(z) &= 1 - \text{erf}(z), \\
    \text{erf}(z) &= \frac{2}{\sqrt{\pi}} \int_{0}^z e^{-t^2}dt.
\end{align}
\begin{table*}[ht!]
\centering
\caption{Convexity constants of the element-wise activation functions in PyTorch}
\label{tab:activation_function_convecities}
\begin{tabular}{|c|c|c|c|c|}
\hline
\textbf{Activation Function} & \textbf{L} & \textbf{m} & \textbf{S} & \textbf{P} \\ \hline 
ELU ($\alpha = 1$) \cite{Clevert2015}& $\max(1, \alpha)$ & 0 & $\max(1, \alpha)$ & 0 \\ \hline 
Hardshrink \cite{Cancino2002} & $\infty$ & 0 & $\infty$ & $\infty$ \\ \hline 
Hardsigmoid \cite{Courbariaux2015}& $\frac{1}{6}$ & 0 & $\frac{1}{6}$ & 0 \\ \hline 
Hardtanh \cite{collobert2004} & 1 & 0 & 1 & 0 \\ \hline 
Hardswish \cite{Howard2019} & 1.5 & -0.5 & 1 & -0.75 \\ \hline 
LeakyReLU ($\alpha = 1e^{-2}$) \cite{Maas2013} & 1 & $\alpha$ & $ 1 + \alpha $ & $\alpha$ \\ \hline 
LogSigmoid & 1 & 0 & 1 & 0 \\ \hline 
PReLU ($\alpha=\frac{1}{4}$) \cite{He2015PReLU}& 1 & $\alpha$ & $1 + \alpha$ & $\alpha$ \\ \hline 
ReLU \cite{McCulloch1943} & 1 & 0 & 1 & 0 \\ \hline 
ReLU6 \cite{Howard2017} & 1 & 0 & 1 & 0 \\ \hline 
RReLU \cite{Xu2015} & $\infty$ & $-\infty$ & $\infty$ & $\infty$ \\ \hline 
SELU \cite{Klambauer2017} &  $\alpha \times \text{scale} \approx  1.758099341$ & 0 & $\alpha \times \text{scale} \approx  1.758099341$  & 0 \\ \hline 
CELU \cite{Barron2017} & 1 & 0 & 1 & 0 \\ \hline 
GELU \cite{Hendrycks2016} &  $\frac{\text{erfc}(1)}{2}-\frac{1}{e \sqrt{\pi }}$ &  $\frac{1}{2} (\text{erf}(1)+1)+\frac{1}{e \sqrt{\pi }}$ & 1 & $\frac{\left(e \sqrt{\pi } (\text{erf}(1)+1)+2\right) \left(e \sqrt{\pi } \text{erfc}(1)-2\right)}{4 e^2 \pi }$ \\
   &  $\approx 1.128904145$ &  $\approx -0.1289041452$ &  & $\approx -0.145520424$ \\ \hline 
Sigmoid \cite{Sak2014} & 1 & 0 & 1 & 0 \\ \hline 
SiLU \cite{Elfwing2017} & 1.099839320 & -0.09983932013 & 1 & -0.1098072100 \\ \hline 
Softplus \cite{Zhou2016} & 1 & 0 & 1 & 0 \\ \hline 
Mish ($\alpha \ge \frac{1}{2}$) \cite{Misra2019} & 1.199678640 & -0.2157287822 & 0.8060623125 & -0.2204297485 \\ \hline 
Softshrink \cite{Cancino2002} & 1 & 0 & 1 & 0 \\ \hline 
Softsign \cite{Ping2017} & 1 & 0 & 1 & 0 \\ \hline 
Tanh \cite{Sak2014} & 1 & 0 & 1 & 0 \\ \hline 
Tanhshrink & 1 & 0 & 1 & 0 \\ \hline 
Threshold & 1 & 0 & 1 & 0 \\ \hline
\end{tabular}
\end{table*}
\section{Experiment}

Based on the designed network constraints, a network was thus generated. To run such a network due to the co-dependence of the $\Lambda_l, L_l$ and $m_l$ from the next layer and the first layer, the evaluation of the network needed to be run in two passes, a backward pass which computed the $\Lambda_l$ and $C_l$ parameters, as illustrated in Figure \ref{fig:backwards_pass}, 
\begin{figure*}
    \centering
    \includegraphics[width=1\linewidth]{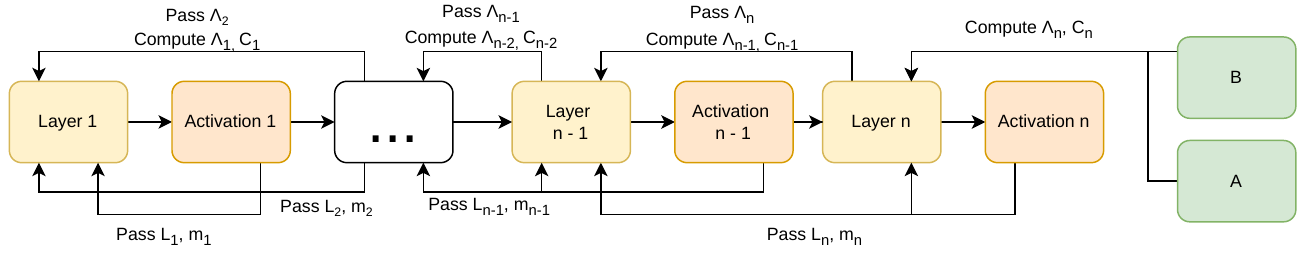}
    \caption{Backwards pass}
    \label{fig:backwards_pass}
\end{figure*}
and then, the forward pass performed the inferences using the computed parameters as a standard residual network as illustrated in \ref{fig:forward_pass}.
\begin{figure*}
    \centering
    \includegraphics[width=1\linewidth]{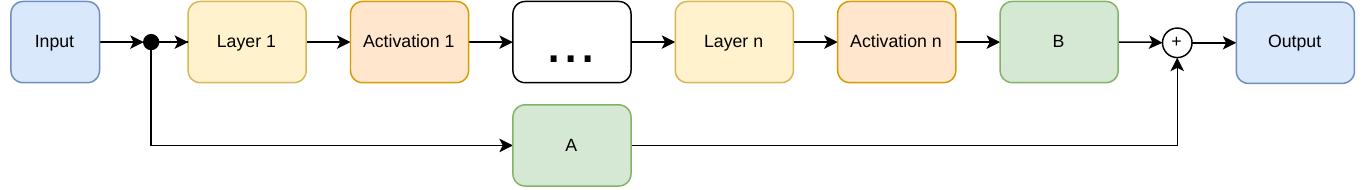}
    \caption{Forward pass}
    \label{fig:forward_pass}
\end{figure*}
It should be noted that it was not possible to make use of techniques such as a batch normalization \cite{Szegedy2016, Chen2017, He2015DeepResidual}, which is a common practice in more modern ResNet architectures. This was because normalization was not a constrained $\mathcal{L}$-Lipschitz formulation as the normalized features were computed as \cite{Ioffe2015, Gouk2021}:
\begin{align}
    \hat{x}^{(k)} &= \frac{x^{(k)} - \text{E}[x^{(k)}]}{\sqrt{\text{Var}[x^{(k)}]}},
\end{align}
Which could be represented as a linear layer where,
\begin{align}
C_b &= \text{diag}\left(\sqrt{\text{Var}[x^{(1)}]}, \cdots, \sqrt{\text{Var}[x^{(d)}]} \right)^{-1}, \\
b_b &= -\begin{bmatrix}
    \frac{\text{E}[x^{(1)}]}{\sqrt{\text{Var}[x^{(1)}]}} & \cdots & \frac{\text{E}[x^{(d)}]}{\sqrt{\text{Var}[x^{(d)}]}}
\end{bmatrix}^\top.
\end{align}
Where it would only be in particular conditions that the batch normalization would follow the constraints posed by Table \ref{tab:matrix_constrains}; this is due to the variance scaling term being very hard to control and is defined by the dataset that is inputted into the system.

To test the network's capabilities, it was initially tested on a straightforward dataset to fit $y = \frac{1}{2} \sin(x)$ on $x \in (-2\pi, 2\pi )$, which is a $\frac{1}{2}$ Lipchitz bounded function as $arg \max_{x} \frac{\partial}{\partial x} \frac{1}{2} \sin(x) = \frac{1}{2} arg \max_{x} \cos(x)  = \frac{1}{2}$, which should thus make it possible to train the network on. However, it was noticed that no matter what optimizer, activation function, size or number of hidden layers, learning rate, or other hyper-parameters used, the system would be unable to fit the function to any degree of accuracy using the MSE loss function. This is illustrated from the output results in Figures \ref{fig:output_optimizers} and \ref{fig:loss_optimizers}.
\begin{figure}[!ht]
    \centering
    \includegraphics[width=1\linewidth]{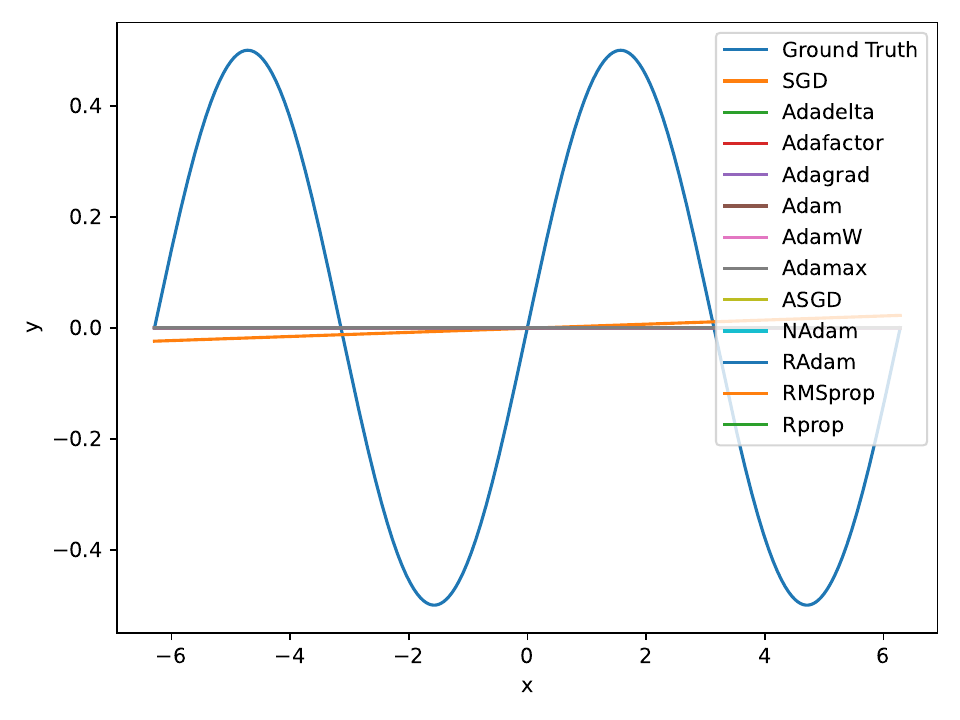}
    \caption{Trained L-Lipschitz network output over multiple optimizers}
    \label{fig:output_optimizers}
\end{figure}
\begin{figure}[!ht]
    \centering
    \includegraphics[width=1\linewidth]{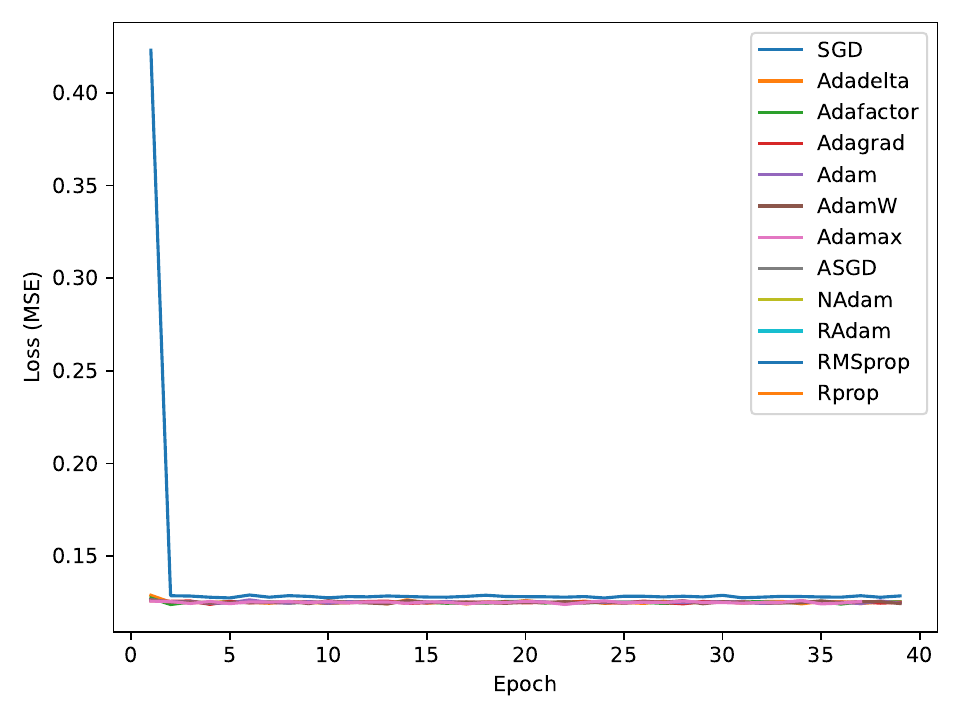}
    \caption{MSE training loss over multiple optimizers}
    \label{fig:loss_optimizers}
\end{figure}
The network's output looked like a single regression line, and no expressiveness of the network could be observed. For a minimum error to fit with this line, it would have to be a line defined as:
\begin{align}
    \mathcal{L}_l &= \frac{1}{4 \pi} \int_{-2\pi}^{2 \pi} (ax + b - \sin(x))^2 dx, \nonumber \\
                &= \frac{1}{6}   \left(4 a \left(2 \pi ^2 a+3\right)+6 b^2+3\right).
\end{align}
Whose minimum would be defined when $a = -\frac{3}{4 \pi^2}$ and $b = 0$, with a total error of $\mathcal{L}_l = \frac{1}{2} - \frac{3}{4 \pi^2}$.  After further inspection of the network, the main culprit in the decay issue was $C_1$'s magnitude as the $\Lambda_1$ magnitude in the network became very large, which caused an over-constraining of the $C_1$ matrix parameter. Given that the Gershgorin circle theorem is only an approximation of the eigenvalue locations, this caused the overall network's compounding approximations to over-constrain the network and thus disable the non-linear portion of the system as such, the network comes the simple $y \approx Ax$ formula, where $A$ is the parameterized diagonal matrix. It is thus sadly noted that this type of network with the current type of parameterization for the weights and biases of the system does not function as a universal function approximator. As such this paper is only able to elaborate on the current methodology for solving the LMI using the Gershgorin circle for more complicated general LMI structures; however, it should be noted that if the LMI follows a more standard matrix structure such as a tri-diagonal form \cite{Xu2024} common in a standard Feedforward Neural Network (FNN) or the likes it is possible to derive more exact eigenvalue constraints on the system.

\section{Conclusion}

This study rigorously derived constraints for the pseudo-tri-diagonal matrix LMI representing a residual network. Given the absence of explicit eigenvalue computations for the tri-diagonal matrix with off-diagonal elements, the Gershgorin circle theorem was employed to approximate the eigenvalue locations of this complex recursive system. The system was decomposed into three distinct blocks, and weight parameterizations were systematically derived and summarized in Table \ref{tab:matrix_constrains}.

A two-step process was detailed once the constraints were established and the network was constructed. Due to the residual network's recursive nature, the normalization parameters needed to be computed and propagated in advance to enable the creation of layer weight parameterizations. This stage was defined as the backward pass. The forward pass involved performing inference on the network.

Upon evaluating the implemented network, it was observed that the Gershgorin circle approximations caused the normalization factors of the inner layers to deactivate the network's non-linear components. Consequently, based on the Gershgorin formulation, the final implementation proved ineffective and unsuitable as a functional approximation. This study establishes a foundation for future research into alternative eigenvalue approximations and refined parameterization strategies, advancing robust deep learning architectures' theoretical and practical development.


\bibliographystyle{IEEEtran}
\bibliography{main}

@article{Araujo2023,
   author = {Alexandre Araujo and Aaron Havens and Blaise Delattre and Alexandre Allauzen and Bin Hu},
   month = {3},
   title = {A Unified Algebraic Perspective on Lipschitz Neural Networks},
   url = {http://arxiv.org/abs/2303.03169},
   year = {2023},
}

@article{Prach2022,
   author = {Bernd Prach and Christoph H. Lampert},
   month = {8},
   title = {Almost-Orthogonal Layers for Efficient General-Purpose Lipschitz Networks},
   url = {https://arxiv.org/abs/2208.03160v2},
   year = {2022},
}

@article{Miyato2018,
   author = {Takeru Miyato and Toshiki Kataoka and Masanori Koyama and Yuichi Yoshida},
   isbn = {1802.05957v1},
   journal = {6th International Conference on Learning Representations, ICLR 2018 - Conference Track Proceedings},
   month = {2},
   publisher = {International Conference on Learning Representations, ICLR},
   title = {Spectral Normalization for Generative Adversarial Networks},
   url = {https://arxiv.org/abs/1802.05957v1},
   year = {2018},
}

@misc{Meunier2022,
   author = {Laurent Meunier and Blaise J Delattre and Alexandre Araujo and Alexandre Allauzen},
   issn = {2640-3498},
   month = {6},
   pages = {15484-15500},
   publisher = {PMLR},
   title = {A Dynamical System Perspective for Lipschitz Neural Networks},
   url = {https://proceedings.mlr.press/v162/meunier22a.html},
   year = {2022},
}

@article{Bartlett2017,
   author = {Peter L. Bartlett and Dylan J. Foster and Matus Telgarsky},
   issn = {10495258},
   journal = {Advances in Neural Information Processing Systems},
   month = {6},
   pages = {6241-6250},
   publisher = {Neural information processing systems foundation},
   title = {Spectrally-normalized margin bounds for neural networks},
   volume = {2017-December},
   url = {https://arxiv.org/abs/1706.08498v2},
   year = {2017},
}

@article{Klambauer2017,
   author = {Günter Klambauer and Thomas Unterthiner and Andreas Mayr and Sepp Hochreiter},
   journal = {Advances in Neural Information Processing Systems},
   title = {Self-Normalizing Neural Networks},
   volume = {30},
   year = {2017},
}

@article{Ioffe2015,
   author = {Sergey Ioffe and Christian Szegedy},
   isbn = {9781510810587},
   journal = {32nd International Conference on Machine Learning, ICML 2015},
   month = {2},
   pages = {448-456},
   publisher = {International Machine Learning Society (IMLS)},
   title = {Batch Normalization: Accelerating Deep Network Training by Reducing Internal Covariate Shift},
   volume = {1},
   url = {https://arxiv.org/abs/1502.03167v3},
   year = {2015},
}

@article{Szegedy2016,
   author = {Christian Szegedy and Sergey Ioffe and Vincent Vanhoucke and Alexander A. Alemi},
   doi = {10.1609/aaai.v31i1.11231},
   issn = {2159-5399},
   journal = {31st AAAI Conference on Artificial Intelligence, AAAI 2017},
   month = {2},
   pages = {4278-4284},
   publisher = {AAAI press},
   title = {Inception-v4, Inception-ResNet and the Impact of Residual Connections on Learning},
   url = {https://arxiv.org/abs/1602.07261v2},
   year = {2016},
}

@misc{Chen2017,
      title={Rethinking Atrous Convolution for Semantic Image Segmentation}, 
      author={Liang-Chieh Chen and George Papandreou and Florian Schroff and Hartwig Adam},
   month = {6},
      year={2017},
      eprint={1706.05587},
      archivePrefix={arXiv},
      primaryClass={cs.CV},
      url={https://arxiv.org/abs/1706.05587}, 
}

@article{He2015DeepResidual,
   author = {Kaiming He and Xiangyu Zhang and Shaoqing Ren and Jian Sun},
   doi = {10.1109/CVPR.2016.90},
   isbn = {9781467388504},
   issn = {10636919},
   journal = {Proceedings of the IEEE Computer Society Conference on Computer Vision and Pattern Recognition},
   month = {12},
   pages = {770-778},
   publisher = {IEEE Computer Society},
   title = {Deep Residual Learning for Image Recognition},
   volume = {2016-December},
   url = {https://arxiv.org/abs/1512.03385v1},
   year = {2015},
}

@article{Clevert2015,
   author = {Djork-Arné Clevert and Thomas Unterthiner and Sepp Hochreiter},
   journal = {4th International Conference on Learning Representations, ICLR 2016 - Conference Track Proceedings},
   month = {11},
   publisher = {International Conference on Learning Representations, ICLR},
   title = {Fast and Accurate Deep Network Learning by Exponential Linear Units (ELUs)},
   url = {http://arxiv.org/abs/1511.07289},
   year = {2015},
}

@article{Cancino2002,
   author = {Hector F. Cancino‐De‐Greiff and Rubersy Ramos‐Garcia and Juan V. Lorenzo‐Ginori},
   doi = {10.1002/cmr.10043},
   issn = {1043-7347},
   issue = {6},
   journal = {Concepts in Magnetic Resonance},
   month = {1},
   pages = {388-401},
   title = {Signal de‐noising in magnetic resonance spectroscopy using wavelet transforms},
   volume = {14},
   url = {https://onlinelibrary.wiley.com/doi/10.1002/cmr.10043},
   year = {2002},
}

@article{Courbariaux2015,
   author = {Matthieu Courbariaux and Yoshua Bengio and Jean Pierre David},
   issn = {10495258},
   journal = {Advances in Neural Information Processing Systems},
   month = {11},
   pages = {3123-3131},
   publisher = {Neural information processing systems foundation},
   title = {BinaryConnect: Training Deep Neural Networks with binary weights during propagations},
   volume = {2015-January},
   url = {https://arxiv.org/abs/1511.00363v3},
   year = {2015},
}

@PHDTHESIS{collobert2004,
         author = {Collobert, Ronan},
       projects = {Idiap},
          title = {Large Scale Machine Learning},
           year = {2004},
         school = {Universit{\'{e}} de Paris VI},
     postscript = {ftp://ftp.idiap.ch/pub/reports/2004/collobert_2004_phdthesis.ps.gz},
ipdmembership={learning},
}

@article{Howard2019,
   author = {Andrew Howard and Mark Sandler and Bo Chen and Weijun Wang and Liang Chieh Chen and Mingxing Tan and Grace Chu and Vijay Vasudevan and Yukun Zhu and Ruoming Pang and Quoc Le and Hartwig Adam},
   doi = {10.1109/ICCV.2019.00140},
   isbn = {9781728148038},
   issn = {15505499},
   journal = {Proceedings of the IEEE International Conference on Computer Vision},
   month = {5},
   pages = {1314-1324},
   publisher = {Institute of Electrical and Electronics Engineers Inc.},
   title = {Searching for MobileNetV3},
   volume = {2019-October},
   url = {https://arxiv.org/abs/1905.02244v5},
   year = {2019},
}

@article{McCulloch1943,
   author = {Warren S. McCulloch and Walter Pitts},
   doi = {10.1007/BF02478259/METRICS},
   issn = {00074985},
   issue = {4},
   journal = {The Bulletin of Mathematical Biophysics},
   month = {12},
   pages = {115-133},
   publisher = {Kluwer Academic Publishers},
   title = {A logical calculus of the ideas immanent in nervous activity},
   volume = {5},
   url = {https://link.springer.com/article/10.1007/BF02478259},
   year = {1943},
}

@article{Maas2013,
   author = {Andew L. Maas and Y. Hannun Awni and Andrew Y Ng},
   journal = {Proceedings of the 30th International Conference on Machine Learning},
   month = {6},
   title = {Rectifier Nonlinearities Improve Neural Network Acoustic Models},
   volume = {28},
   year = {2013},
}

@article{He2015PReLU,
   author = {Kaiming He and Xiangyu Zhang and Shaoqing Ren and Jian Sun},
   journal = {CoRR},
   month = {2},
   title = {Delving Deep into Rectifiers: Surpassing Human-Level Performance on ImageNet Classification},
   volume = {abs/1502.01852},
   url = {http://arxiv.org/abs/1502.01852},
   year = {2015},
}

@article{Xu2015,
   author = {Bing Xu and Naiyan Wang and Tianqi Chen and Mu Li},
   month = {5},
   title = {Empirical Evaluation of Rectified Activations in Convolutional Network},
   url = {http://arxiv.org/abs/1505.00853},
   year = {2015},
}

@article{Howard2017,
   author = {Andrew G. Howard and Menglong Zhu and Bo Chen and Dmitry Kalenichenko and Weijun Wang and Tobias Weyand and Marco Andreetto and Hartwig Adam},
   month = {4},
   title = {MobileNets: Efficient Convolutional Neural Networks for Mobile Vision Applications},
   url = {http://arxiv.org/abs/1704.04861},
   year = {2017},
}

@article{Barron2017,
   author = {Jonathan T. Barron},
   month = {4},
   title = {Continuously Differentiable Exponential Linear Units},
   url = {http://arxiv.org/abs/1704.07483},
   year = {2017},
}

@article{Hendrycks2016,
   author = {Dan Hendrycks and Kevin Gimpel},
   month = {6},
   title = {Gaussian Error Linear Units (GELUs)},
   url = {https://arxiv.org/abs/1606.08415v5},
   year = {2016},
}

@article{Sak2014,
   author = {Haşim Sak and Andrew Senior and Françoise Beaufays},
   month = {2},
   title = {Long Short-Term Memory Based Recurrent Neural Network Architectures for Large Vocabulary Speech Recognition},
   url = {https://arxiv.org/abs/1402.1128v1},
   year = {2014},
}

@article{Elfwing2017,
   author = {Stefan Elfwing and Eiji Uchibe and Kenji Doya},
   doi = {10.1016/j.neunet.2017.12.012},
   issn = {18792782},
   journal = {Neural Networks},
   month = {2},
   pages = {3-11},
   pmid = {29395652},
   publisher = {Elsevier Ltd},
   title = {Sigmoid-Weighted Linear Units for Neural Network Function Approximation in Reinforcement Learning},
   volume = {107},
   url = {https://arxiv.org/abs/1702.03118v3},
   year = {2017},
}

@article{Zhou2016,
   author = {Mingyuan Zhou},
   month = {8},
   title = {Softplus Regressions and Convex Polytopes},
   url = {http://arxiv.org/abs/1608.06383},
   year = {2016},
}

@article{Misra2019,
   author = {Diganta Misra},
   journal = {31st British Machine Vision Conference, BMVC 2020},
   month = {8},
   publisher = {British Machine Vision Association, BMVA},
   title = {Mish: A Self Regularized Non-Monotonic Activation Function},
   url = {http://arxiv.org/abs/1908.08681},
   year = {2019},
}

@article{Ping2017,
   author = {Wei Ping and Kainan Peng and Andrew Gibiansky and Sercan O. Arik and Ajay Kannan and Sharan Narang and Jonathan Raiman and John Miller},
   journal = {6th International Conference on Learning Representations, ICLR 2018 - Conference Track Proceedings},
   month = {10},
   publisher = {International Conference on Learning Representations, ICLR},
   title = {Deep Voice 3: Scaling Text-to-Speech with Convolutional Sequence Learning},
   url = {http://arxiv.org/abs/1710.07654},
   year = {2017},
}

@article{Xu2024,
   author = {Yuezhu Xu and S. Sivaranjani},
   month = {4},
   title = {ECLipsE: Efficient Compositional Lipschitz Constant Estimation for Deep Neural Networks},
   url = {https://arxiv.org/abs/2404.04375v2},
   year = {2024},
}

@article{He2015ResNet,
   author = {Kaiming He and Xiangyu Zhang and Shaoqing Ren and Jian Sun},
   doi = {10.1109/CVPR.2016.90},
   isbn = {9781467388504},
   issn = {10636919},
   journal = {Proceedings of the IEEE Computer Society Conference on Computer Vision and Pattern Recognition},
   month = {12},
   pages = {770-778},
   publisher = {IEEE Computer Society},
   title = {Deep Residual Learning for Image Recognition},
   volume = {2016-December},
   url = {https://arxiv.org/abs/1512.03385v1},
   year = {2015},
}

@article{Zagoruyko2016,
   author = {Sergey Zagoruyko and Nikos Komodakis},
   doi = {10.5244/C.30.87},
   journal = {British Machine Vision Conference 2016, BMVC 2016},
   month = {5},
   pages = {87.1-87.12},
   publisher = {British Machine Vision Conference, BMVC},
   title = {Wide Residual Networks},
   volume = {2016-September},
   url = {https://arxiv.org/abs/1605.07146v4},
   year = {2016},
}

@article{Hu2017,
   author = {Jie Hu and Li Shen and Samuel Albanie and Gang Sun and Enhua Wu},
   doi = {10.1109/TPAMI.2019.2913372},
   issn = {19393539},
   issue = {8},
   journal = {IEEE Transactions on Pattern Analysis and Machine Intelligence},
   month = {9},
   pages = {2011-2023},
   pmid = {31034408},
   publisher = {IEEE Computer Society},
   title = {Squeeze-and-Excitation Networks},
   volume = {42},
   url = {https://arxiv.org/abs/1709.01507v4},
   year = {2017},
}

@article{Xie2016,
   author = {Saining Xie and Ross Girshick and Piotr Dollár and Zhuowen Tu and Kaiming He},
   doi = {10.1109/CVPR.2017.634},
   isbn = {9781538604571},
   journal = {Proceedings - 30th IEEE Conference on Computer Vision and Pattern Recognition, CVPR 2017},
   month = {11},
   pages = {5987-5995},
   publisher = {Institute of Electrical and Electronics Engineers Inc.},
   title = {Aggregated Residual Transformations for Deep Neural Networks},
   volume = {2017-January},
   url = {https://arxiv.org/abs/1611.05431v2},
   year = {2016},
}

@inbook{10.5555/3454287.3455312,
author = {Fazlyab, Mahyar and Robey, Alexander and Hassani, Hamed and Morari, Manfred and Pappas, George J.},
title = {Efficient and accurate estimation of lipschitz constants for deep neural networks},
year = {2019},
publisher = {Curran Associates Inc.},
address = {Red Hook, NY, USA},
booktitle = {Proceedings of the 33rd International Conference on Neural Information Processing Systems},
articleno = {1025},
numpages = {12}
}

@article{Tsuzuku2018,
  author       = {Yusuke Tsuzuku and
                  Issei Sato and
                  Masashi Sugiyama},
  title        = {Lipschitz-Margin Training: Scalable Certification of Perturbation
                  Invariance for Deep Neural Networks},
  journal      = {CoRR},
  volume       = {abs/1802.04034},
  year         = {2018},
  url          = {http://arxiv.org/abs/1802.04034},
  eprinttype    = {arXiv},
  eprint       = {1802.04034},
  bibsource    = {dblp computer science bibliography, https://dblp.org}
}

@article{Inkawhich2019,
   author = {Matthew Inkawhich and Yiran Chen and Hai Li},
   isbn = {9781450375184},
   issn = {15582914},
   journal = {Proceedings of the International Joint Conference on Autonomous Agents and Multiagent Systems, AAMAS},
   month = {5},
   pages = {557-565},
   publisher = {International Foundation for Autonomous Agents and Multiagent Systems (IFAAMAS)},
   title = {Snooping Attacks on Deep Reinforcement Learning},
   volume = {2020-May},
   url = {https://arxiv.org/abs/1905.11832v2},
   year = {2019},
}

@article{Goodfellow2014,
   author = {Ian J. Goodfellow and Jonathon Shlens and Christian Szegedy},
   journal = {3rd International Conference on Learning Representations, ICLR 2015 - Conference Track Proceedings},
   month = {12},
   publisher = {International Conference on Learning Representations, ICLR},
   title = {Explaining and Harnessing Adversarial Examples},
   url = {https://arxiv.org/abs/1412.6572v3},
   year = {2014},
}

@article{Sandryhaila2013,
   author = {Aliaksei Sandryhaila and Jose M. F. Moura},
   month = {6},
   title = {Eigendecomposition of Block Tridiagonal Matrices},
   url = {https://arxiv.org/abs/1306.0217v1},
   year = {2013},
}

@article{Agarwal2019,
   author = {Etika Agarwal and S. Sivaranjani and Vijay Gupta and Panos Antsaklis},
   doi = {10.23919/ACC.2019.8814701},
   isbn = {9781538679265},
   issn = {07431619},
   journal = {Proceedings of the American Control Conference},
   month = {7},
   pages = {5816-5821},
   publisher = {Institute of Electrical and Electronics Engineers Inc.},
   title = {Sequential synthesis of distributed controllers for cascade interconnected systems},
   volume = {2019-July},
   year = {2019},
}

@article{Meunier2021,
   author = {Laurent Meunier and Blaise Delattre and Alexandre Araujo and Alexandre Allauzen},
   issn = {26403498},
   journal = {Proceedings of Machine Learning Research},
   month = {10},
   pages = {15484-15500},
   publisher = {ML Research Press},
   title = {A Dynamical System Perspective for Lipschitz Neural Networks},
   volume = {162},
   url = {https://arxiv.org/abs/2110.12690v2},
   year = {2021},
}

@article{Gouk2021,
   author = {Henry Gouk and Eibe Frank and Bernhard Pfahringer and Michael J. Cree},
   doi = {10.1007/s10994-020-05929-w},
   issn = {0885-6125},
   issue = {2},
   journal = {Machine Learning},
   month = {2},
   pages = {393-416},
   title = {Regularisation of neural networks by enforcing Lipschitz continuity},
   volume = {110},
   url = {http://link.springer.com/10.1007/s10994-020-05929-w},
   year = {2021},
}

@article{Aziznejad2020,
   author = {Shayan Aziznejad and Harshit Gupta and Joaquim Campos and Michael Unser},
   doi = {10.1109/TSP.2020.3014611},
   journal = {IEEE Transactions on Signal Processing},
   month = {1},
   pages = {4688-4699},
   publisher = {Institute of Electrical and Electronics Engineers Inc.},
   title = {Deep Neural Networks with Trainable Activations and Controlled Lipschitz Constant},
   volume = {68},
   url = {http://arxiv.org/abs/2001.06263 http://dx.doi.org/10.1109/TSP.2020.3014611},
   year = {2020},
}

@article{Bear2024,
   author = {Jay Bear and Adam Prügel-Bennett and Jonathon Hare},
   month = {10},
   title = {Rethinking Deep Thinking: Stable Learning of Algorithms using Lipschitz Constraints},
   url = {https://arxiv.org/abs/2410.23451v1},
   year = {2024},
}

@article{Kumar2017,
   author = {Siddharth Krishna Kumar},
   month = {4},
   title = {On weight initialization in deep neural networks},
   url = {https://arxiv.org/abs/1704.08863v2},
   year = {2017},
}

\appendices

\end{document}